%% file: BOBWBandits_nocomments.tex
\documentclass{article}
\usepackage{graphicx}
\usepackage{natbib} 
\usepackage{amssymb}
\usepackage{amsmath}
\usepackage{amsthm}
\usepackage{fullpage}
\usepackage{url}
\usepackage{xcolor}
\usepackage{algorithm}
\usepackage{algorithmic}
\usepackage{bibentry}
\usepackage{mathtools}
\usepackage{anyfontsize}

\input{symbol.tex}

\newtheorem{theorem}{Theorem}
\newtheorem{lemma}[theorem]{Lemma}

\title{A Best-of-Both-Worlds Proof for Tsallis-INF without Fenchel Conjugates}
\author{
Wei-Cheng Lee\\
KAUST\\
{\tt\small weicheng.lee@kaust.edu.sa}
\and
Francesco Orabona\\
KAUST\\
{\tt\small francesco@orabona.com}
}

\nobibliography* 
\begin{document}
\maketitle

\begin{abstract}
	In this short note, we present a simple derivation of the best-of-both-world guarantee for the Tsallis-INF multi-armed bandit algorithm from \bibentry{ZimmertS21}. In particular, the proof uses modern tools from online convex optimization and avoid the use of conjugate functions. Also, we do not optimize the constants in the bounds in favor of a slimmer proof.
\end{abstract}

\section{Introduction}

The multi-armed bandit problem is a classic framework for sequential decision-making under uncertainty, encapsulating the fundamental trade-off between exploration and exploitation. This problem is typically studied in two primary settings: the stochastic setting, where the loss for each arm is drawn independently and identically from a fixed distribution, and the adversarial setting, where an arbitrary and potentially adaptive sequence of losses is generated. While numerous algorithms have been developed that are optimal for one of these settings, a significant challenge has been to design a single algorithm that performs well in both, without prior knowledge of the environment. Such an algorithm is said to have a ``best-of-both-worlds" guarantee.

By choosing an appropriate regularization function, a Follow-The-Regularized-Leader (FTRL)~\citep{Gordon99, Gordon99b,ShalevS06,ShalevS06b,AbernethyHR08,HazanK08} approach can achieve low regret in the adversarial setting while adapting to the easier stochastic setting to attain near-optimal, logarithmically-scaling regret, that is, a best-of-both-worlds guarantee. The Tsallis-INF algorithm~\citep{AudibertB09}, when used in a FTRL algorithm was shown to have this property by \citet{ZimmertS21}.

The aim of this note is to  provide a concise and self-contained proof of this guarantee for the Tsallis-INF algorithm. In particular, the derivation relies on standard and modern techniques from online convex optimization, particularly the FTRL regret lemma and its local norm analysis. This approach simplifies the original proof by avoiding the use of conjugate duality, making the analysis more direct and accessible. Also, for the sake of clarity and brevity, the focus is on the core arguments rather than optimizing the absolute constants in the final regret bounds.

\section{Setting and Definitions}

In this section, we formally introduce the multi-armed bandit problem and define the key metrics for evaluating an algorithm's performance in both the adversarial and stochastic settings.

\paragraph{The Multi-Armed Bandit Problem.}
The multi-armed bandit problem is a sequential decision-making framework. An agent, or player, is faced with a set of $d$ possible actions, often referred to as ``arms" or ``experts." The process unfolds over a series of $T$ discrete time steps, or rounds.

The protocol for each round $t = 1, \dots, T$ is as follows:
\begin{itemize}
	\item Action Selection: The agent selects a probability distribution $\bx_t = (x_{t,1}, \dots, x_{t,d})$ from the $(d-1)$-dimensional probability simplex, $\Delta^{d-1} = \{\bx \in \R^d : x_i \ge 0, \sum_{i=1}^d x_i = 1\}$.
	\item Drawing an Arm: The agent draws an arm $I_t \in \{1, \dots, d\}$ according to the distribution $\bx_t$, \emph{i.e.}, $\Pr\{I_t=i\} = x_{t,i}$.
	\item Loss Observation: The environment reveals a loss vector $\bg_t = (g_{t,1}, \dots, g_{t,d})$, where each $g_{t,i} \in [0, G]$ represents the loss associated with arm $i$ in round $t$. Note that we will assume that $G$ is known to the algorithm.
	\item Bandit Feedback: The agent observes \emph{only} the loss for the selected arm, $g_{t,I_t}$. The losses of the other arms, $g_{t,i}$ for $i \neq I_t$, remain unknown to the agent.
	\item Incurring Loss: The agent incurs the observed loss $g_{t,I_t}$.
\end{itemize}

The agent's objective is to choose a sequence of distributions $\bx_1, \dots, \bx_T$ to minimize its total cumulative loss over the $T$ rounds.

\paragraph{Performance Metrics: Regret.}
To evaluate the performance of an algorithm, we compare its expected total loss to that of the best single arm in hindsight. This difference is known as \textbf{regret}. The specific definition of regret depends on how the loss vectors $\bg_t$ are generated.

\paragraph{Adversarial Setting.}
In the adversarial setting, the sequence of loss vectors $\bg_1, \dots, \bg_T$ can be chosen arbitrarily, possibly by an adversary who knows the agent's algorithm. The goal is to perform well against the worst-case sequence. The \textbf{adversarial regret} is defined as
\[
	\Regret_T := \sum_{t=1}^T g_{t,I_t} - \min_{i \in \{1, \dots, d\}} \ \sum_{t=1}^T g_{t,i}
\]
In the analysis we will upper bound the expected value of this quantity, $\E[\Regret_T]$, where the expectation $\E[\cdot]$ is taken over the algorithm's internal randomization (the drawing of $I_t \sim \bx_t$). An algorithm is considered effective in this setting if its expected regret grows sublinearly with $T$, typically $\mathcal{O}(\sqrt{T})$.

\paragraph{Stochastic Setting.}
In the stochastic setting, the losses are not adversarial but are generated from fixed probability distributions. Specifically, for each arm $i$, the losses $g_{t,i}$ are drawn independently and identically (i.i.d.) from an unknown distribution $\rho_i$ with a well-defined mean $\mu_i = \E[g_{t,i}]$.

We define $\mu^\star = \min_{i \in \{1, \dots, d\}} \mu_i$. For any suboptimal arm $i$, we define its suboptimality gap as $\Delta_i = \mu_i - \mu^\star$.

In this setting, performance is measured by the \textbf{pseudo-regret}, which compares the agent's expected loss to the expected loss of an oracle that always pulls the optimal arm:
\[
	\PRegret_T
	:= \E\left[\sum_{t=1}^T g_{t,I_t}\right] - \min_{i \in \{1, \dots, d\}} \ \E\left[\sum_{t=1}^T g_{t,i}\right]
	= \E\left[\sum_{t=1}^T g_{t,I_t}\right] - T\mu^\star~.
\]
Here, the expectation is over both the algorithm's randomization and the randomness from the loss distributions $\{\rho_i\}$. A good stochastic algorithm achieves a regret that grows logarithmically with $T$, \emph{i.e.}, $\mathcal{O}(\log T)$.

Observe that in the adversarial case the losses are deterministic, so the second expectation can be removed and pseudo-regret becomes equivalent to the expected regret defined before.

\paragraph{Best-of-Both-Worlds Guarantee.}
An algorithm is said to have a \emph{best-of-both-worlds guarantee} if it can achieve $\mathcal{O}(\sqrt{T})$ regret in the adversarial setting and $\mathcal{O}(\ln T)$ regret in the stochastic setting, without needing to know in advance which environment it is operating in.

\paragraph{Notation.}
The inner product between two vectors $\bx, \by\in \R^d$ will be denoted by $\langle \bx, \by\rangle$.
We will denote by $B_\psi(\bx;\by):=\psi(\bx)-\psi(\by)-\langle \nabla (\by), \bx-\by\rangle$ where is the \emph{Bregman divergence} of $\bx$ around $\by$ with respect to a function $\psi$. We will denote the weighted norm of a vector $\bx$ with respect to a PD matrix $\bA$ as $\|\bx\|_{\bA}=\sqrt{\langle \bA \bx, \bx\rangle}$.

\section{Main Result}
\begin{algorithm}[h]
	\caption{Tsallis-INF with FTRL}
	\label{alg:tsallis}
	\begin{algorithmic}[1]
		{
		\REQUIRE{$G>0$}
		\FOR{$t=1$ {\bfseries to} $T$}
		\STATE{$\bx_{t} = \argmin_{\bx \in \Delta^{d-1}} \ \langle \sum_{n=1}^{t-1} \tilde{\bg}_n, \bx\rangle - 4 G \sqrt{t} \sum_{i=1}^d \sqrt{x_{i}}  $}
		\STATE{Draw $I_t$ according to $\Pr\{I_t=i\}=x_{t,i}$}
		\STATE{Select arm $I_t$}
		\STATE{Observe \emph{only} the loss of the selected arm $g_{t,I_t}$ and pay it}
		\STATE{Construct the estimate $\tilde{g}_{t,i}=\begin{cases} \frac{g_{t,i}}{x_{t,i}}, & i=I_t\\ 0, & \text{otherwise}\end{cases}$ for $i=1,\dots,d$}
		\ENDFOR
		}
	\end{algorithmic}
\end{algorithm}

\begin{theorem}
	\label{thm:main}
	Assume that $0\leq g_{t,i}\leq G$, for all $t=1, \dots, T$, $i=1, \dots, d$, where $d\geq 1$. Then, Algorithm~\ref{alg:tsallis} satisfies
	\[
		\E[\Regret_T]
		\leq 32 G \sqrt{(d-1) T}~.
	\]
	Moreover, if in addition the $g_{t,i}$ are i.i.d. from a distribution $\rho_i$ with mean $\mu_i$ for $i=1,\dots, d$, and  $\argmin_i \mu_i$ is unique, then we also have
	\[
		\PRegret_T
		\leq 256 G^2 \sum_{i: \mu_i\neq \mu^*} \frac{1+\ln T}{\Delta_i}~.
	\]
\end{theorem}

For the proof of this theorem we will use the following two results.
The first one is a classic result and we report its proof for completeness.
\begin{lemma}
	\label{lemma:stoch_band_decomp}
	For any policy of selection of the arms, the pseudo-regret is equal to
	\[
		\PRegret_T = \sum_{i=1}^d \E[S_{T,i}] \Delta_i,
	\]
	where $S_{T,i}$ is the total number of times we pull arm $i$, for $i=1, \dots,d$.
\end{lemma}
\begin{proof}
	Observe that
	\[
		\sum_{t=1}^T g_{t,I_t}
		= \sum_{t=1}^T \sum_{i=1}^d g_{t,i} \boldsymbol{1}[I_t=i]~.
	\]
	Hence,
	\begin{align*}
		\PRegret_T
		 & = \E\left[\sum_{t=1}^T g_{t,I_t}\right] - T \mu^\star
		= \E\left[\sum_{t=1}^T (g_{t,I_t}- \mu^\star)\right]
		= \sum_{i=1}^d \sum_{t=1}^T \E[\boldsymbol{1}[I_t=i](g_{t,i}-\mu^\star)]               \\
		 & = \sum_{i=1}^d \sum_{t=1}^T \E[\E[\boldsymbol{1}[I_t=i](g_{t,i}- \mu^\star)|I_t]]
		= \sum_{i=1}^d \sum_{t=1}^T \E[\boldsymbol{1}[I_t=i] \E[g_{t,i} - \mu^\star |I_t]]     \\
		 & = \sum_{i=1}^d \sum_{t=1}^T \E[\boldsymbol{1}[I_t=i]] (\mu_i- \mu^\star)~. \qedhere
	\end{align*}
\end{proof}

The other result we need is a FTRL regret upper bound with local norms.
\begin{lemma}[{\citep[Lemma 7.16]{Orabona19}}]
	\label{lemma:ftrl_local_norms}
	Let $\psi_1, \dots, \psi_T :\mathcal{X} \to \R$ be a sequence of twice differentiable regularization functions with the Hessian positive definite in the interior of their domains. Let $\mathcal{V} \subseteq \mathcal{X} \subseteq \R^d$ be a closed and non-empty set, and $\bg_t \in \R^d$ for $t=1, \dots, T$ an arbitrary sequence of vectors.
	Denote by $F_t(\bx) = \psi_{t}(\bx) + \sum_{i=1}^{t-1} \langle \bg_i, \bx\rangle$.
	Assume that $\argmin_{\bx \in \mathcal{V}} \ F_{t}(\bx)$ is not empty and set $\bx_t \in \argmin_{\bx \in \mathcal{V}} \ F_{t}(\bx)$ and assume it to be in the interior of the domain.
	Assume that $\tilde{\bx}_{t+1} := \argmin_{\bx \in \R^d} \ \langle \bg_t, \bx\rangle + B_{\psi_t}(\bx; \bx_t)$ exists. Then, there exist $\bz_t$ on the line segments between $\bx_t$ and $\tilde{\bx}_{t+1}$ for $t=1,\dots, T$, such that the following inequality holds for any $\bu \in \mathcal{V}$
	\[
		\sum_{t=1}^T \langle \bg_t,\bx_t - \bu\rangle
		\leq \psi_{T+1}(\bu) - \min_{\bx \in \mathcal{V}} \ \psi_{1}(\bx)
		+ \sum_{t=1}^T \left[\frac{\|\bg_t\|^2_{(\nabla^2 \psi_t(\bz_t))^{-1}}}{2} + \psi_t(\bx_{t+1}) - \psi_{t+1}(\bx_{t+1})\right]~.
	\]
\end{lemma}
\begin{proof}
	First of all, by summing both sides of the equality, it is easy to verify that
	\[
		\sum_{t=1}^T \langle \bg_t, \bx_t-\bu\rangle
		= \psi_{T+1}(\bu) - \min_{\bx \in \mathcal{V}} \ \psi_{1}(\bx) + \sum_{t=1}^T [F_t(\bx_t) - F_{t+1}(\bx_{t+1}) + \langle\bg_t,\bx_t\rangle] + F_{T+1}(\bx_{T+1}) - F_{T+1}(\bu),
	\]
	where the sum of the last two terms on the right hand side is negative by the definition of $\bx_{T+1}$.

	Now, observe that
	\[
		F_t(\bx_t) - F_{t+1}(\bx_{t+1}) +  \langle\bg_t,\bx_t\rangle
		= F_t(\bx_t) - F_{t}(\bx_{t+1}) +  \langle\bg_t,\bx_t\rangle -  \langle\bg_t,\bx_{t+1}\rangle + \psi_t(\bx_{t+1}) - \psi_{t+1}(\bx_{t+1})~.
	\]
    We now focus on the first four terms on the right hand side of this equality.
    First, observe that $\psi_t$ are strictly convex because the Hessians are positive definite. Hence, they can be used to define Bregman divergences.
	Moreover, from the optimality condition of $\bx_t$, we have
	$\langle \nabla F_t(\bx_t), \bv - \bx_t\rangle \geq 0$, forall $\bv \in \mathcal{V}$.
	Hence, in particular we have
	$
		\langle \nabla F_t(\bx_t), \bx_{t+1} - \bx_t\rangle \geq 0
	$.
	Using this inequality, we have
	\[
		B_{F_t}(\bx_{t+1};\bx_t)
		= F_t(\bx_{t+1}) - F_t(\bx_t) - \langle \nabla F_t(\bx_t), \bx_{t+1} - \bx_t\rangle
		\leq F_t(\bx_{t+1}) - F_t(\bx_t)~.
	\]
	This last inequality implies that
	\begin{align*}
		F_t(\bx_t) - F_{t}(\bx_{t+1}) +  \langle\bg_t,\bx_t\rangle -  \langle\bg_t,\bx_{t+1}\rangle
		 & \leq \langle \bg_t, \bx_t - \bx_{t+1}\rangle - B_{F_t}(\bx_{t+1};\bx_t)~.
	\end{align*}
	Bregman divergences are independent of linear terms, so $B_{F_t}(\bx_{t+1};\bx_t)=B_{\psi_t}(\bx_{t+1};\bx_t)$.
	From the Taylor's theorem, we have that $B_{\psi_t}(\bx_{t+1};\bx_t) = \frac{1}{2}(\bx_{t+1}-\bx_t)^\top \nabla^2 \psi_t(\bz_t) (\bx_{t+1}-\bx_t)$, where $\bz_t$ is on the line segment between $\bx_t$ and $\bx_{t+1}$. Observe that this is $\frac12 \|\bx_{t+1}-\bx_t\|^2_{\nabla^2 \psi_t(\bz_t)}$ and it is indeed a norm because we assumed the Hessian of $\psi_t$ to be positive definite.
	Now observe that given that the Hessian is PD, we have that
	\begin{align*}
		\langle \bg_t, \bx_t-\tilde{\bx}_{t+1}\rangle
		&= \langle (\nabla^2 \psi_t(\bz_t))^{-1} \bg_t, \nabla^2 \psi_t(\bz_t)(\bx_t-\tilde{\bx}_{t+1})\rangle
		\leq \|(\nabla^2 \psi_t(\bz_t))^{-1} \bg_t\|_2 \cdot \|\nabla^2 \psi_t(\bz_t)(\bx_t-\tilde{\bx}_{t+1})\|_2\\
		&\leq \frac{1}{2}\|\bg_t\|^2_{(\nabla^2 \psi_t(\bz_t))^{-1}} + \frac{1}{2}\|\tilde{\bx}_{t+1}-\bx_t\|^2_{\nabla^2 \psi_t(\bz_t)},
	\end{align*}
	where we used Cauchy–Schwarz inequality and the inequality between arithmetic and geometric means between two numbers, \emph{i.e.}, $ab \leq \frac12 a^2 + \frac12 b^2$.
	Hence, we have
	\begin{align*}
		F_t(\bx_t) - F_{t}(\bx_{t+1}) +  \langle\bg_t,\bx_t\rangle -  \langle\bg_t,\bx_{t+1}\rangle
		 & \leq \langle \bg_t, \bx_t - \bx_{t+1}\rangle - B_{\psi_t}(\bx_{t+1};\bx_t)                                                                                            \\
		 & \leq \max_{\bx \in \R^d} \ \langle \bg_t, \bx_t - \bx\rangle - B_{\psi_t}(\bx;\bx_t)                                                                                  \\
		 & = \langle \bg_t, \bx_t - \tilde{\bx}_{t+1}\rangle - B_{\psi_t}(\tilde{\bx}_{t+1};\bx_t)                                                                               \\
		 & \leq \frac{1}{2}\|\bg_t\|^2_{(\nabla^2 \psi_t(\bz_t))^{-1}} + \frac{1}{2}\|\tilde{\bx}_{t+1}-\bx_t\|^2_{\nabla^2 \psi_t(\bz_t)} - B_{\psi_t}(\tilde{\bx}_{t+1};\bx_t) \\
		 & =\frac{1}{2}\|\bg_t\|^2_{(\nabla^2 \psi_t(\bz_t))^{-1}}~. \qedhere
	\end{align*}
\end{proof}

\begin{proof}[Proof of Theorem~\ref{thm:main}]
	Algorithm~\ref{alg:tsallis} is an instantiation of FTRL with linear losses $\ell_t(\bx)=\langle \tilde{\bg}_t, \bx\rangle$ and regularizer $\psi_t(\bx)=-\lambda_t \psi(\bx)$, where $\psi(\bx)= -\sum_{i=1}^d \sqrt{x_i}$ and $\lambda_t = 4 G \sqrt{t}$.
	Note that $\lambda_{T+1}$ has no influence on the regret, so we can set it equal to $\lambda_T$ \emph{ex post facto}.
	Moreover, given that $\psi$ has partial derivative that goes $-\infty$ if a coordinate goes to 0, we know from the first-order optimality condition that $\bx_t$ cannot have any coordinate equal to 0 for any $t$.
	Hence, from Lemma~\ref{lemma:ftrl_local_norms} we have
	\[
		\sum_{t=1}^T \langle \tilde{\bg}_t,\bx_t - \bu\rangle
		\leq \psi_{T}(\bu) - \min_{\bx \in \mathcal{V}} \ \psi_{1}(\bx)
		+ \frac12 \sum_{t=1}^T \|\tilde{\bg}_t\|^2_{(\nabla^2 \psi_t(\bz_t))^{-1}} + \sum_{t=1}^{T-1}(\psi_t(\bx_{t+1}) - \psi_{t+1}(\bx_{t+1})),
	\]
	where $\bz_t$ is on the line segment between $\bx_t$ and $\tilde{\bx}_{t+1}:=\argmin_{\bx \in \R^d_{\geq 0}} \ B_{\psi_t}(\bx; \bx_t) +\langle \tilde{\bg}_t, \bx\rangle$.
	Note that $\bx_1=\argmin_{\bx \in \Delta^{d-1}} \ \psi_1(\bx)=[1/d, \dots, 1/d]$ and $\psi(\bx_1)=-\sqrt{d}$. Also, for $\be_j$ being the $j$-th canonical basis, we have $\psi(\be_j)=-1$ for any $j=1, \dots, d$.

	We now use the fact that update of the algorithm remain the same if $b_t$ is subtracted to all the losses in each round $t$. Hence, we can state the bound for a choice of $b_t$ that helps reduce the first sum in the above inequality.
	The Hessian of the $\psi_t$ is the diagonal matrix with entries $(\nabla^2 \psi(\bx))_{ii}=\frac{1}{4 x_i^{1.5}}$. Hence, we have
	\[
		\sum_{t=1}^T \langle \tilde{\bg}_t, \bx_t-\bu\rangle
		\leq \psi_{T}(\bu)-\psi_1(\bx_1)+ 2\sum_{t=1}^T \sum_{i=1}^d (\tilde{g}_{t,i}-b_t)^2 z_{t,i}^{1.5} + \sum_{t=1}^{T-1} (\psi_t(\bx_{t+1})-\psi_{t+1}(\bx_{t+1}) ), \qquad \forall \bu \in \Delta^{d-1}~.
	\]

	If we constrain $0\leq b_t\leq G$ and $\frac{G}{\lambda_t}\leq 1/4$, we have that $\tilde{x}_{t+1,i}$ has the following closed-form expression:
	\[
		\tilde{x}_{t+1,i}
		=\frac{x_{t,i}}{(1+2\frac{1}{\lambda_t}(\tilde{g}_{t,i}-b_t)x_{t,i}^{1/2})^2}
		\leq 4x_{t,i}~.
	\]
	Then, we select $b_t=\boldsymbol{1}\{I_t=j\} g_{t,j}$, where $j$ is an arbitrary arm, and we take expectation to have
	\begin{align*}
		\E\left[\sum_{i=1}^d (\tilde{g}_{t,i}-\boldsymbol{1}\{I_t=j\}g_{t,j})^2 z_{t,i}^{1.5}\right]
		&\leq 8\E\left[\sum_{i=1}^d (\tilde{g}_{t,i}-\boldsymbol{1}\{I_t=j\}g_{t,j})^2 x_{t,i}^{1.5}\right]\\
		 & = 8 \E\left[\sum_{i=1}^d \left(\frac{g_{t,i}}{x_{t,i}}\boldsymbol{1}\{I_t=i\}-\boldsymbol{1}\{I_t=j\} g_{t,j}\right)^2 x_{t,i}^{1.5} \right]                                                                                             \\
		 & = 8\E\left[g_{t,j}^2 \left(\frac{1}{x_{t,j}}-1\right)^2 x_{t,j}^{1.5} \boldsymbol{1}\{I_t=j\} + \sum_{i\neq j} \left(\frac{g^2_{t,i}}{x^2_{t,i}}\boldsymbol{1}\{I_t=i\}+\boldsymbol{1}\{I_t=j\} g^2_{t,j}\right) x_{t,i}^{1.5}\right] \\
		 & = 8\E\left[g_{t,j}^2 \left(\frac{1}{x_{t,j}}-1\right)^2 x_{t,j}^{1.5} x_{t,j} + \sum_{i\neq j} g_{t,i}^2 x_{t,i}^{0.5}+ x_{t,j} \sum_{i\neq j} g^2_{t,j} x_{t,i}^{1.5} \right]                                                         \\
		 & \leq 24 G^2 \E\left[ \sum_{i \neq j} \sqrt{x_{t,i}} \right],
	\end{align*}
	where in the inequality we used $\left(\frac{1}{x_{t,j}}-1\right)^2 x_{t,j}^{1.5} x_{t,j} = x_{t,j}^{0.5} (1-x_{t,j})^2\leq 1-x_{t,j} \leq \sum_{i\neq j} \sqrt{x_{t,j}}$ and $x_{t,j} \sum_{i\neq j} g^2_{t,j} x_{t,i}^{1.5}\leq G^2 \sum_{i\neq j} x_{t,i}^{0.5}$.

	Putting all together, setting $\bu=\be_k$, for all $k=1, \dots, d$, we have
	\begin{align*}
		\E\left[\sum_{t=1}^T \langle \bg_t, \bx_t-\be_k\rangle\right]
		 & \leq -\lambda_{T} -\lambda_1 \psi(\bx_1) + 48 G^2 \sum_{t=1}^T \frac{1}{\lambda_t}\sum_{i\neq j} \E[\sqrt{x_{t,i}}]
		+ \sum_{t=1}^{T-1} \left(\lambda_{t+1} -\lambda_t\right)(-\E[\psi( \bx_{t+1})])                                            \\
		 & =\lambda_1 (-\psi(\bx_1)-1) + 48 G^2 \sum_{t=1}^T \frac{1}{\lambda_t}\sum_{i\neq j} \E[\sqrt{x_{t,i}}]
		+ \sum_{t=1}^{T-1} \left(\lambda_{t+1} -\lambda_t\right)(-\E[\psi( \bx_{t+1})]-1)                                          \\
		 & \leq \lambda_1 \sum_{i\neq j} \sqrt{x_{1,i}} + 48 G^2 \sum_{t=1}^T \frac{1}{\lambda_t}\sum_{i\neq j} \E[\sqrt{x_{t,i}}]
		+ \sum_{t=1}^{T-1} \left(\lambda_{t+1} -\lambda_t\right)\sum_{i\neq j} \E[\sqrt{x_{t+1,i}}]\\
		& = 48 G^2 \sum_{t=1}^T \frac{1}{\lambda_t}\sum_{i\neq j} \E[\sqrt{x_{t,i}}]
		+ \sum_{t=0}^{T-1} \left(\lambda_{t+1} -\lambda_t\right)\sum_{i\neq j} \E[\sqrt{x_{t+1,i}}],
	\end{align*}
	where in the first equality we have used the fact that $\sum_{t=1}^{T-1} (\lambda_{t+1}-\lambda_t)=\lambda_{T}-\lambda_1$, in the second inequality we used that $-1+\sum_{i=1}^d \sqrt{x_{t+1,i}}\leq \sum_{i\neq j} \sqrt{x_{t+1,i}}$ for any $j=1, \dots, d$, and in the last equality we set $\lambda_0=0$.
	Using the fact that $\sqrt{t}-\sqrt{t-1}\leq \frac{1}{\sqrt{t}}$ for $t=1, 2, \dots$, we have
	\begin{equation}
		\label{eq:proof_bobw_1}
		\E\left[\sum_{t=1}^T \langle\bg_t, \bx_t-\be_k\rangle\right]
		\leq
		G \sum_{t=1}^{T} \left(4\sqrt{t}-4\sqrt{t-1}+\frac{12}{\sqrt{t}}\right) \sum_{i\neq j} \E[\sqrt{x_{t,i}}]
		\leq G \sum_{t=1}^{T} \frac{16}{\sqrt{t}} \sum_{i\neq j} \E[\sqrt{x_{t,i}}]~.
	\end{equation}
	Using the fact that $\sum_{i\neq j} \sqrt{x_{t,i}} \leq \sqrt{d-1}$ by Cauchy-Schwarz inequality and $\sum_{t=1}^T \frac{1}{\sqrt{t}}\leq 2 \sqrt{T}$, we get the first stated bound.

	We now move to bound the pseudo-regret. Let $j$ be the index of the best arm.
	From Lemma~\ref{lemma:stoch_band_decomp} and the fact that the expected time we pull arm $i$ is $\sum_{t=1}^T \E[x_{t,i}]$, we have that $\E[\Regret_T(\be_j)]=\PRegret_T=\sum_{t=1}^T \sum_{i\neq j} \E[x_{t,i}] \Delta_i$.
	So, setting $k=j$, from \eqref{eq:proof_bobw_1} we have
	\begin{align*}
		\PRegret_T
		&= \E[\Regret_T(\be_j)]
		= 2\E[\Regret_T(\be_j)] -\sum_{t=1}^T \sum_{i\neq j} \E[x_{t,i}] \Delta_i\\
		&\leq  \E\left[\sum_{t=1}^T \sum_{i\neq j} \left( 32G\frac{\sqrt{x_{t,i}}}{\sqrt{t}} - x_{t,i} \Delta_i\right) \right]
		  \leq  G^2 \sum_{t=1}^T \sum_{i\neq j} \frac{256}{t \Delta_i}
		\leq  256 G^2  \sum_{i\neq j} \frac{1+\ln T}{ \Delta_i},
	\end{align*}
	where in the first inequality we used $a\sqrt{x}-b x\leq \frac{a^2}{4b}$ for any $a,b>0$.
\end{proof}

\section{Differences with the Result by \citet{ZimmertS21}}

Our proof is a direct distillation of the ideas in \citet{ZimmertS21}, including Lemma~\ref{lemma:ftrl_local_norms}, and we claim no novelty for it. Yet, there are a few important differences to point out.
\begin{itemize}
	\item \citet{ZimmertS21} call their algorithm Online Mirror Descent (OMD), but this is technically wrong even if it is a common mistake (see discussion in \citet[Section 7.13]{Orabona19}). One can see that the update weights all the $\tilde{\bg}_t$ in the same way, while rescaling their sum by $\lambda_t$. This behavior is the one of FTRL, while OMD would rescale each $\tilde{\bg}_t$ by $\lambda_t$.
	      Once one recognizes the algorithm as FTRL, it is natural to use the known results for it.
	\item We show the (very minor) improvement in the dependency from $d$ to $d-1$, that makes sense because with $d=1$ one should suffer no regret.
	\item Our constants are much worse, because we favored a simpler proof over tighter constants.
	\item They consider multiple estimators $\tilde{\bg}_t$, while we focus on the commonly used one of the importance weighted estimator. They also consider different powers for the Tsallis entropy, while we focus on the one that gives the best bound.
	\item As mentioned at the beginning, our proof avoid completely the use of Fenchel conjugates. In particular, this greatly simplify the proof of the local norm bound.
\end{itemize}

\section*{Acknowledgments}
We thank Yevgeny Seldin and Julian Zimmert for feedback and comments on a preliminary version of this manuscript.

\bibliography{learning}

\end{document}

%% file: symbol.tex
\renewcommand{\Pr}{\field{P}}

\newcommand{\be}{\boldsymbol{e}}
\newcommand{\bg}{\boldsymbol{g}}

\newcommand{\bx}{\boldsymbol{x}}
\newcommand{\bu}{\boldsymbol{u}}
\newcommand{\by}{\boldsymbol{y}}

\newcommand{\bA}{\boldsymbol{A}}

\newcommand{\bz}{\boldsymbol{z}}

\newcommand{\bv}{\boldsymbol{v}}

\newcommand{\argmin}{\mathop{\mathrm{argmin}}}

\newcommand{\field}[1]{\mathbb{#1}}

\newcommand{\R}{\field{R}}

\newcommand{\E}{\field{E}}

\DeclareMathOperator{\PRegret}{P-Regret}
\DeclareMathOperator{\Regret}{Regret}